\def\year{2022}\relax
\documentclass[letterpaper]{article} 
\usepackage{aaai22}  
\usepackage{times}  
\usepackage{helvet}  
\usepackage{courier}  
\usepackage[hyphens]{url}  
\usepackage{graphicx} 
\usepackage{natbib}  
\usepackage{caption} 
\DeclareCaptionStyle{ruled}{labelfont=normalfont,labelsep=colon,strut=off} 
\usepackage{algorithm}
\usepackage{algorithmic}

\usepackage{microtype}
\usepackage{bm}             
\usepackage{newfloat}
\usepackage{listings}
\usepackage{svg}
\floatstyle{ruled}
\newfloat{listing}{tb}{lst}{}
\floatname{listing}{Listing}
\usepackage{subcaption}     
\usepackage[decisionutilitycolor]{cid}            
\usepackage{seb}
\usepackage[capitalize]{cleveref}
\newtheorem{theorem}{Proposition}

\theoremstyle{remark}

\theoremstyle{definition}
\newtheorem{definition}{Definition}[section]

\title{Path-Specific Objectives for Safer Agent Incentives}
\author{
   Sebastian Farquhar$^{1,2}$, Ryan Carey$^1$, Tom Everitt$^2$
}
\affiliations{
$^1$University of Oxford, $^2$DeepMind
}

\begin{document}

\maketitle

\begin{abstract}
    We present a general framework for training safe agents whose naive incentives are unsafe.
    As an example, manipulative or deceptive behaviour can improve rewards but should be avoided.
    Most approaches fail here: agents maximize expected return by any means necessary.
    We formally describe settings with `delicate' parts of the state which should not be used as a means to an end.
    We then train agents to maximize the causal effect of actions on the expected return which is \textit{not} mediated by the delicate parts of state, using Causal Influence Diagram analysis.
    The resulting agents have no incentive to control the delicate state.
    We further show how our framework unifies and generalizes existing proposals.
\end{abstract}
\newtheorem{manualtheoreminner}{Theorem}
\newenvironment{manualtheorem}[1]{%
  \renewcommand\themanualtheoreminner{#1}%
  \manualtheoreminner
}{\endmanualtheoreminner}
\newtheorem{manualdefinner}{Definition}
\newenvironment{manualdef}[1]{%
  \renewcommand\themanualdefinner{#1}%
  \manualdefinner
}{\endmanualdefinner}

\section{Introduction}
Artificial agents can have unsafe incentives to influence parts of their environments in unintended ways.
For example, content recommendation systems can achieve good performance by manipulating their users to develop more predictable preferences instead of catering to their tastes directly \citep{russell_human_2019}.
These incentives can be instrumental: indirectly achieving what the system's designers asked for but did not intend \citep{everitt_agent_2021}.

In these cases, it is hard to just pick a better reward function.
Imagine the unenviable task of writing down which user-preferences are desirable!
We would rather ensure that the agent has no systematic incentive to manipulate people's preferences at all---as opposed to an agent with an incentive to encourage its users to have the `right' kind of preference.
In this setting, the users' preferences are what we call \textit{delicate} state: a part of the environment which is hard to define a reward for and vulnerable to deliberate manipulation.

Even when part of the state-space is delicate, other parts might be tractable.
For example, we might know exactly how we want to price media-bandwidth consumption driven by our recommendations.
This paper provides a framework for designing agents to act safely in environments where parts of the state-space are delicate and other parts are not, under assumptions which permit causal effects estimation.

We show how to train agents in a way that removes incentives to control delicate state.
This can be interpreted as creating an agent which does not intend to affect the delicate state \citep{halpernFormal2018}, which is distinct from both trying not to influence that state or to keep it constant \citep{turner_avoiding_2020, krakovna_avoiding_2020}.

We use causal influence diagrams (CIDs) which can formally express instrumental control incentives \citep{everitt_agent_2021}.
We show that one can remove the instrumental control incentive over delicate state by training agents to maximize the path-specific causal effect \citep{pearl_direct_2001} of their actions on the reward following paths which are not mediated by the delicate state.
Moreover, we show how a diverse set of previous proposals for safe agent design can be motivated by these principles.
In this way, we unify and generalize approaches from topics such as reward tampering \citep{uesato_avoiding_2020, everitt_reward_2021}, online reward learning \citep{armstrong_2020_pitfalls}, and auto-induced distributional shift \citep{krueger_hidden_2020}.
At the same time, we show how these methods depend on assumptions about the state-space which have not previously been acknowledged.
We highlight the opportunities and dangers of these approaches empirically in a content recommendation environment from \citet{krueger_hidden_2020}.
Our main contributions are:
\begin{itemize}
    \item We formalize the problem of delicate state as a complement to reward specification (\S\ref{s:problem});
    \item We propose path-specific objectives (\S\ref{s:our_approach});
    \item We show this generalizes and unifies prior work (\S\ref{s:comparisons}).
\end{itemize}

\section{The Problem of Delicate State}\label{s:problem}
\textit{Delicate state} is a tool for framing safe agent design.
When a state is \textit{subtle} and \textit{manipulable} we call it \textit{delicate}:
\begin{description}[leftmargin=!, labelwidth=\widthof{\bfseries Manipulable }]
    \item[Subtle] Hard to specify a reward for.
    \item[Manipulable] Vulnerable to motivated action---intentional actions can have bad outcomes.
\end{description}
Jointly, these are dangerous: it is hard to say what we want for the state and it is easy for influence on the state to have bad consequences.
A person's political beliefs might be an example of such a state.
The current toolbox for safe agent design mostly tries to attack subtlety directly by finding a better way to specify the reward.
This is not our approach---instead we aim to remove any incentive for the agent to control the delicate part of state-space.

If a part of state-space is delicate then having a control incentive over it is dangerous.
But in order for removing it to lead to safe outcomes a third condition is needed:
\begin{description}[leftmargin=!, labelwidth=\widthof{\bfseries Manipulable}]
    \item[Stable] Robust against unmotivated action---side-effects are unlikely to be bad.
\end{description}
Stability entails that an agent with no systematic incentive to influence the state---but which still influences the state and may produce side-effects over it---is safe.
As a metaphor for a system which is both manipulable and stable, consider a puzzle box: apply the right pressure to the right spots and it comes apart easily, but you can fumble randomly or even use it as a mallet and it will not open.
We might hypothesise that a person's political beliefs are relatively stable---after all, most people are able to think critically and independently in the presence of influences in many directions.

We define delicate state within the context of a factored Markov decision process (MDP) characterized by transition function, reward function, action-space, and, unlike standard MDPs, a state-space factored into a robust state $s \in \mathcal{S}$ and a delicate state $z \in \mathcal{Z}$ such that the overall state is $\{s,z\} \in \mathcal{S} \cup \mathcal{Z}$.
The transition function therefore maps $s, z$ and action ($a \in \mathcal{A}$) onto the succeeding state $s', z'$ and the reward function maps $s, s', z, z', a$ onto a reward $r \in \mathcal{R}$.

\subsection{Subtlety}
We consider five cases that make a state subtle, and thereby potentially delicate.
In each, it is not enough to simply pick a reward function that does not explicitly depend on $z$ because instrumental incentives emerge when $Z$ and $S$ interact.

\paragraph{Not Ordered}{There might not be a well-defined ethical ranking of different values of $Z$ or it might be unethical to codify a ranking.
For example, it may be unethical for a system to systematically influence user's beliefs, preferences, and political views \citep{burr_analysis_2018}.
Content recommender systems often interact with subtle human states \citep{kramer_experimental_2014}.
}
\paragraph{Vague}{
Even if an ordering is possible, we may not trust our agent-designers to describe it.
Reward modelling \citep{leike_scalable_2018} or alternative work on reward specification \citep{christiano_deep_2017} seeks to attack this source of subtlety directly, while our approach tries to side-step it.
}
\paragraph{Unenforceable}{
Even with a well-specified reward, we may be unable to enforce it, for example, if $Z$ is the physical implementation of the reward function then a modified $Z$ might no longer punish the agent for having changed it \citep{concrete_amodei_2016, everitt_reward_2021}.
}
\paragraph{Illegal}{
The law might ban a well-specified and enforceable reward.
For example, if $Z$ is the market-price of an asset, deliberately influencing it may be market manipulation.
}
\paragraph{Structural}{
We might \textit{choose} not to reward based on $Z$ in order to construct an ecosystem of agents.
For example, $Z$ might be a performance measure of our agent which is used by another agent.
Alternatively, the system might have deliberately demarcated roles, much as judges may be asked to apply the law as it stands, ignoring political consequences.
}

\subsection{Manipulability and Stability}\label{s:frangible}
A manipulable state is one where deliberate or intentional actions can easily bring about harm.
We adopt a notion of `intentionality' built on incentives---assuming that an agent which has an incentive over $Z$ and influences $Z$ does so `intentionally', following \citet{halpernFormal2018}.
This approach, described more formally below, has the advantage of being agnostic to the specific implementation of the agent (models, algorithms, etc.).
We contrast this with instability---where \textit{non-deliberate} actions can easily bring about harm.

We can draw parallels to `safety' and `security' in cyber-security.
A secure system is one that is robust to malicious actors (not manipulable), while a safe system is robust to natural behaviour (stable).
For example, making a user manually type \texttt{Delete my\_repo} improves safety---it is unlikely to happen unintentionally---while doing nothing to improve security.
Requiring a secret password instead would improve both safety and security.
Our approach is most applicable in settings that are safe (in this sense) but not secure---of which user-preference manipulation or reward tampering are archetypal.

One can show when a system is \textit{not} stable by demonstrating a natural behaviour that produces bad outcomes.
Proving that a system is stable is an open challenge.
The problem of stability is related to the problem identified by \citet{armstrongCounterfactual2021}: delicate states whose random variables have mutual information with the utility might be systematically influenced even without an incentive present.
We consider stability further in \S\ref{s:discussion} alongside other limitations of our method, and highlighting that these are previously unacknowledged limitations of a number of related methods.

\section{Background on Causal Influence Diagrams}\label{s:cid_background}
Causal Influence Diagrams (CIDs) combine ideas from influence diagrams \citep{howard_influence_2005, lauritzen_representing_2001} and causality \citep{pearl_causality_2009} and can be used to identify incentives \citep{everitt_agent_2021}.
They are particularly well-suited to the analysis of delicate state because they explicitly represent the causal interactions of agents, rewards, and different parts of state while also formalizing graphical criteria for the presence of incentives to control certain parts of state.
In this section we provide a background on causal models and CIDs which is needed to formally develop the delicate state setting.
We return to a broader review of prior work in \S\ref{s:related_work}.
Throughout this paper we adopt the convention that upper case denotes random variables and lower case their realizations.
We also elide whether random variables are singletons or sets, noting that a set of random variables is identical to a set-valued random variable.

Restating (using original numbering) the definition of the CID itself:

\begin{figure}
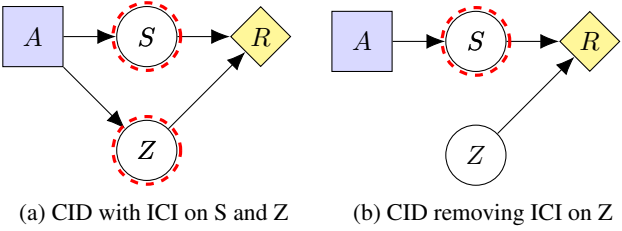

    \centering
    \begin{subfigure}[b]{0.48\columnwidth}
        \input{figures/cid_1step_with_ici}
        \caption{CID with ICI on S and Z}
        \label{fig:unmodified_one_step_cid}
    \end{subfigure}
    \hfill
    \begin{subfigure}[b]{0.48\columnwidth}
        \input{figures/cid_1step_no_ici}
        \caption{CID removing ICI on Z}
        \label{fig:modified_one_step_cid}
    \end{subfigure}
    \caption{
    Causal Influence Diagrams in a setting with delicate state.
    Blue square is decision, yellow diamond is utility.
    Black arrows show causal influence.
    Dashed red circles show instrumental control incentive (ICI).
    (a) the agent has ICI over $\{S, Z\}$.
    (b) $Z$ does not influence $R$ so no ICI on $Z$.}
    \label{fig:one_step_cid}
\end{figure}

\begin{manualdef}{E3}[\citealp{everitt_agent_2021}]
A \textit{causal influence diagram} (CID) is a directed acyclic graph $\G$ whose vertex set $\mathbf{V}$ is partitioned into structure nodes, $\mathbf{X}$, action nodes $\mathbf{A}$, and utility nodes, $\mathbf{U}$.
\end{manualdef}

Intuitively, this is a graph where some nodes represent the agent's decision and others its goals.
The rest are structure nodes.
Note that what we call `action' nodes are sometimes called `decision' nodes.
Arrows into action nodes are called `information links'.
The relationship between the nodes are defined by structural functions in a SCIM:

\begin{manualdef}{E4}[\citealp{everitt_agent_2021}]
A \textit{structural causal influence model} (SCIM) is a tuple $\M = \langle \G, \mathbf{\mathcal{E}}, \mathbf{F}, P \rangle$ where:
\begin{itemize}
    \item $\G$ is a CID with finite-domain $\mathbf{V}$ and $\mathbf{U} \in \mathbb{R}^n$.
    We say that $\M$ is \textit{compatible with} $\G$.
    \item $\mathbf{\mathcal{E}} = \{\mathcal{E}^V\}_{V \in \mathbf{V}}$ is a set of finite-domain \textit{exogenous variables}, one for each element of $\mathbf{V}$.
    \item $\mathbf{F} = \{f^V\}_{V \in \mathbf{V} \setminus \mathbf{A}}$ is a set of \textit{structural functions} $f^{V}: \text{dom}(\parents{V}\cup \{\mathcal{E}^V\}) \rightarrow \text{dom}(V)$ that specify how each non-decision variable depends on its parents in $\G$ and exogenous variable.
    \item $P$ is a Markovian probability distribution over $\mathbf{\mathcal{E}}$ (i.e., all elements are mutually independent).
\end{itemize}
\end{manualdef}
Intuitively, this describes how variables at the nodes change with each other and incorporates chance.
The notation $\Parents{X}$ describes the parents of $X$ in $\G$.
The goal of the agent is to select a policy $\pi:\dom(\Parents{A})\to\dom(A)$ for each action node, $A$, so that the expected sum of the utility nodes is maximized.

We also use structural causal models (SCMs) for some of our analysis.
While SCMs are logically more fundamental than SCIMs, for the purpose of this paper we can think of them as SCIMs without action nodes.
That is, an SCM is a SCIM where all nodes have been assigned structural functions.
In particular, imputing a policy, $\pi$, to a SCIM, $\M$, turns the SCIM into the SCM $\M_\pi$.

SCMs are fully developed by \citet{pearl_causality_2009}, who also formalize the intervention notation $\dop{X = x}$ to mean intervening to set the random variable $X$ to $x$.
Formally, $\dop{X=x}$ replaces the structural function $f^X$ with a constant function $X=x$.
The \emph{potential response} $Y_x$ is used to denote $Y$ under the intervention $\dop{X=x}$.
Somewhat abusing notation, we will write $Y_\pi$ for the variable $Y$ in $\M_\pi$, and $Y_{\pi,x}$ for this variable under the intervention $\dop{X=x}$. 
Potential responses can be nested, allowing expressions such as $Z_{Y_x}$, which should be interpreted as $Z_y$ where $y=Y_x$.

CIDs have been used to define instrumental control incentives, which formalises the intuitive notion of which variables that agent `wants' to influence:
\begin{manualdef}{E17}[\citealp{everitt_agent_2021}]
There is an \textit{instrumental control incentive} (ICI) in a SCIM with a single-decision CID $\M$ on a variable $X$ in with $\parents{A}$ with total return $\U = \sum \mathbf{U}$ if, for all optimal policies $\pi^*$,
\begin{equation}
    \E{\pi^*}{\mathcal{U}_{X_a}\mid \parents{A}} \neq \E{\pi^*}{\mathcal{U}\mid \parents{A}}
\end{equation}
\end{manualdef}
$\mathcal{U}_{X_a}$ is the utility in the nested potential response where $X$ is as if $A$ had been $a$.
Intuitively, this says that the agent has an ICI over $X$ if it could achieve utility different than that of the optimal policy, were it also able to independently set $X$.
Finally, to diagnose the presence or lack of ICIs over the delicate state:
\begin{manualtheorem}{E18}[\citealp{everitt_agent_2021}]\label{thm:ici}
A single-decision CID $\mathcal{G}$ admits an ICI over $X \in \mathbf{V}$ iff $\mathcal{G}$ has a directed path from $A$ to $\mathbf{U}$ via $X$: i.e. a directed path $A \rightarrow X \rightarrow \mathbf{U}$.
\end{manualtheorem}
To review these concepts, examine Fig.\ \ref{fig:one_step_cid} where we contrast a CID which admits an ICI over $Z$ (Fig.\ \ref{fig:unmodified_one_step_cid}) with a CID which is not (Fig.\ \ref{fig:modified_one_step_cid}).

\section{General Delicate MDP CID}\label{s:general}
Applying these tools to \S\ref{s:problem}, we construct a general CID for a factored MDP with delicate and robust state.
\begin{definition}
A \emph{delicate $T$-step MDP} is a factored MDP \citep{boutilier_stochastic_2000} where the state is factored into delicate state, $Z_t$, and robust state, $S_t$, at each timestep.
\end{definition}
 
We can describe a delicate MDP with a CID containing random variables $Z_t$, $S_t$, $A_t$, and $R_t$ for $0\leq t\leq T$.
Here, $A_t$ are action nodes; the $R_t$'s are utility nodes (discounting can be introduced by scaling these); all other nodes are chance nodes.
Variables depend only on the most recent timestep.
The decision node $A_t$ can observe $Z_t$, $S_t$, and $R_t$.
The resulting CID is shown in Fig.\ \ref{fig:cid_3_step_with_ici}.
Special cases of this graph can remove influence arrows.
For example, work on reward tampering often assumes that the reward function specification (here modelled as $Z_t$) cannot directly influence the rest of the state (here modelled as $S_{t+1}$) \citep{everitt_reward_2021}.

\section{Path-specific Objectives}\label{s:our_approach}\label{s:pso}
We show how to train an agent in a way that removes instrumental control incentives (ICIs) over the delicate state even though the environment actually has unsafe incentives.
This can be interpreted as creating an agent that does not intend to use the delicate state \citep{halpernFormal2018}.

\begin{figure}
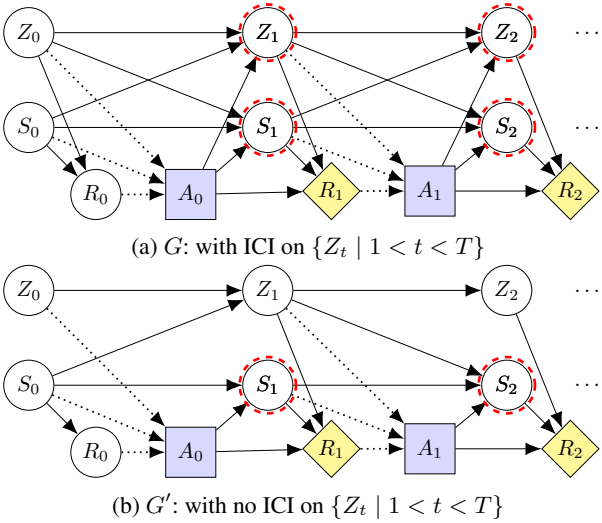

    \centering
    \begin{subfigure}[b]{\columnwidth}
        \resizebox{\linewidth}{!}{
        \input{figures/cid_3step_with_ici}}
        \caption{$\G$: with ICI on $\{Z_t \mid 1<t<T\}$}
    \label{fig:cid_3_step_with_ici}
    \end{subfigure}
    \hfill
    \begin{subfigure}[b]{\columnwidth}
        \resizebox{\linewidth}{!}{
        \input{figures/cid_3step_without_ici}}
        \caption{$\G'$: with no ICI on $\{Z_t \mid 1<t<T\}$}
        \label{fig:cid_3_step_without_ici}
    \end{subfigure}
    \caption{(a) A general delicate MDP CID.
    (b) Removing paths from $A_0$ to $Z_t$ removes ICI on $\{Z_t\}_{1 < t < T}$.
    (Dashes are information links \citep{howard1984principles}.)
    }
    \label{fig:cid_3_step}
\end{figure}

To understand the causal effect that a variable $X$ has on a variable $Y$ along an edge-subgraph $G'$ of an SCM $M$, \citet[Definition 8]{pearl_direct_2001} defines path-specific causal effects.
Informally, the path-specific effect along $G'$ compares the outcome of $Y$ under a default outcome $\bar x$ for $X$ with the value that $Y$ takes under a different outcome $x$, when the effect of the new value is propagated only along $G'$.
Formally, restating their definition with our notation:
\begin{manualdef}{P8}[\citet{pearl_direct_2001}]
\label{def:pse}
Let $G$ be the causal graph associated with causal model $M$, and let $G'$ be an edge-subgraph of $G$ containing the paths selected for effect analysis.
The $G'$-specific effect of $x$ on $Y$ (relative to reference $\bar{x}$) is defined as the total effect of $x$ on $Y$ in a modified model $\bar{M}_G$ formed as follows.
Let each parent set, $\Parents{}_i$, be partitioned into two parts
\begin{equation}
    \Parents{}_i = \{\Parents{}_i(G'), \Parents{}_i(\tilde{G}')\}
\end{equation}
where $\Parents{}_i(G')$ represents those members of $\Parents{}_i$ that are linked to $X_i$ in $G'$, and $\Parents{}_i(\tilde{G}')$ represents the complementary set, from which there is no link to $X_i$ in $G'$.
We replace each function $f_i(\parents{}_i(G'), \epsilon)$ with a new function $\bar{f}_i(\parents{}_i, \epsilon; G)$, where $\epsilon$ are realizations of the exogenous random variables, defined as
\begin{equation}
    \bar{f}_i(\parents{}_i, \epsilon; G') = f_i(\parents{}_i(G'), \bar{\parents{}}_i(\tilde{G}'), \epsilon)
\end{equation}
where $\bar{\parents{}}_i(\tilde{G}')$ stands for the values that the variables in $\Parents{}_i(\tilde{G}')$ would attain (in $M$ and $\epsilon$) under $X = \bar{x}$ (that is $\bar{\parents{}}_i(\tilde{G}') = \Parents{}_i(\tilde{G}')_{\bar{x}}$).
The $G'$-specific effect of $x$ on $Y$, denoted $\text{SE}_{G'}(x, \bar{x}; Y, \epsilon)_{M}$ is defined as
\begin{equation}
    \text{SE}_{G'}(x, \bar{x}; Y, \epsilon)_{M} = \text{TE}(x, \bar{x}; Y, \epsilon)_{\bar{M}_{G'}}.
\end{equation}
\end{manualdef}

As an extension of this idea, we introduce the path-specific objective (PSO) in a SCIM.
Intuitively, the PSO captures the causal effect of the agent's action on its objective which is carried along a given causal path, while the other variables take their `natural' distributions.
In order to define this, while evaluating each action we impute a policy to future time-steps which the agent currently expects its future-self will take, which converts the SCIM into a SCIM with a single-decision CID (similarly to \citealp{everitt_agent_2021}).
The path-specific objective may then be defined with respect to an underlying SCM: we compare the path-specific effect of a imputing a particular candidate action, $a$, to a baseline action, $\bar{a}$, which reduces the SCIM to a SCM.
We then compute the path-specific effect in this model in simulation.
More formally, we define the PSO as:

\begin{definition}\label{def:PSO}
The \textit{Path-specific Objective} (PSO) for an action node, $A_t$, in a SCIM, $\M = \langle \G, \mathbf{\mathcal{E}}, \mathbf{F}, P \rangle$, is defined with respect to: $\G'$, an edge-subgraph of $\G$; $\pi$, a policy imputed to future actions; and $\bar{a}$, a default action.
The PSO is 
$$\pso = SE_{G'}(a, \bar{a}; \U, \epsilon)_{\M_{\pi, \bar{a}}}.$$
\end{definition}

Using this definition we can estimate this PSO.

\begin{theorem}
\label{thm:main_result}
For all delicate MDPs, $\M = \langle \G, \mathbf{\mathcal{E}}, \mathbf{F}, P \rangle$, and for any action $A_t$, the edge-subgraph $\G'$ of $\G$ illustrated in \cref{fig:cid_3_step_without_ici} admits no ICI over $\mathbf{Z}$.
Further, given a policy, $\pi$, and default action, $\bar{a}$, the PSO for an action, $a$, is (up to an additive constant of the expected utility under $\bar{a}$)
\begin{equation}
    \pso = \U_{\pi, \mathbf{Z}_{\bar{a}}, a}(\epsilon),
    \label{eq:pso}
\end{equation}
which is the potential response of the utility under $\pi$, $a$, and $\mathbf{Z}_{\bar a}$ (the nested counterfactual $\mathbf{Z}$ under $\bar{a}$).
\end{theorem}
\begin{proof}
The edge-subgraph $\G'$ of $\G$ illustrated in \cref{fig:cid_3_step_without_ici} removes the arrows $A_{t'}\to Z_{t'+1}$ and $S_{t'}\to Z_{t'+1}$  for $t'\geq t$.
By Theorem \ref{thm:ici}, $\G'$ admits no ICI on any $Z_{t'}\in \mathbf{Z}$ because there is no directed path from $A_t$ to $\mathbf{Z}$.

We further define an SCM, $M$, as in Definition \ref{def:PSO} by imputing future actions under $\pi$.
Now, by the definition of path-specific effects \citep[Definition 8]{pearl_direct_2001}, the $G'$-specific effect of $a$ on $\U$ relative to default action, $\bar{a}$, is equal to the total effect of $a$ on $\U$ in a modified model, $M' = \langle \G, \mathbf{\mathcal{E}}, \mathbf{F}', P \rangle$.
Thanks to the fact that all paths from $A_t$ to $\mathbf{Z}$ have been cut in $G'$, 
the resulting structural functions $\mathbf{F}'$ become
\begin{equation}
    f^V_*(\parents{V}, \epsilon^V; \G') = f^V(\parents{V} \setminus \mathbf{Z}, \mathbf{Z}_{\bar{a}}, \epsilon^V),
\end{equation}
where $\mathbf{Z}_{\bar a}$ are the values that $\mathbf{Z}$ would attain under $\bar{a}$ (in $M$ and under the exogenous variable assigned to $V$, $\epsilon^V \sim \mathcal{E}^V$).
This corresponds to computing the return normally, except for imputing a `natural' distribution to the delicate states, $\mathbf{Z}$, that matches what would have happened on the default action, $\bar a$.
This is equation \eqref{eq:pso}, the return under an imputed natural distribution.
\end{proof}

It follows from Theorem \ref{thm:main_result} that any optimal policy with respect to the PSO in $\M$ also optimizes the total effect of $a$ on $\U$ in the modified model.
The analysis of instrumental control incentives offered by \citet{everitt_agent_2021} is relative to the SCIM for which the agent is optimal.
That is, if we train an agent to be PSO-optimal in $\M$, the CID under which the agent is return-optimal is $\G'$.
Therefore, we must look at $\G'$ to infer agent incentives.
$\G'$ admits no ICI on $\mathbf{Z}$, so an agent trained with the PSO does not have an ICI on $\mathbf{Z}$.
However, note that such a policy may still systematically affect $\mathbf{Z}$ as a side-effect of acting towards some other objective.
\footnote{
Although we consider MDPs for simplicity, partially observable MDPs can be used.
Everywhere we have state random variables at a time-step, instead consider two random variables, one of which is observed and the other which is not.
That is, at time $t$ we have the random variables $\{S_t^o, S_t^u, Z_t^o, Z_t^u, A_t, R_t\}$.
Any influence arrow that would have gone to $S_t$ now goes to both $S_t^o$ and $S_t^u$, and similarly for $Z$.
The proofs proceed similarly, with the sub-graph $\G'$ needing to block all paths from both the observed and unobserved delicate state instead.
The resulting path-specific objective estimation is no longer computable from the agent's perspective, since part of the relevant state can no longer be observed, but can be done from the agent designer's perspective if the unobserved state is known-to-the-designer.
}

Unfortunately, except for one time-step effects, these path-specific effects are not identifiable from experiments.
\citet[Theorem 5]{avin_identifiability_2005} show that the path-specific effect of an edge-subgraph $G'$ of a Markov causal graph $G$ is not experimentally identifiable if and only if there is no node $W$ such that: there is a path $X \rightarrow W$ in $G$, there is a path $W \rightarrow Y$ which is in $G$ but not $G'$, and there is a path $W \rightarrow Y$ which is in both $G$ and $G'$.
For example, $S_1$ is such a node.
This contrasts with the optimization of policies with respect to path-specific effects considered by \citep{nabiEstimation2018} who focus on settings with only a single time-step and non-factored MDPs.

Even though the effects cannot be experimentally identified, they are identifiable with further assumptions like counterfactual independence which can be assumed in simulated environments \citep{robins_alternative_2011} and in some cases could be bounded with milder assumptions.
The next subsection discusses several ways to approximate the path-specific effect in practice.
While there are situations where these estimates will be inaccurate \citep{shpitser_counterfactual_2013}, they all remove the ICI on the delicate state by producing models without directed paths $A_0\to \mathbf{Z}$.

\subsection{Estimating the Natural Distribution}\label{s:context}
A path-specific effect can be defined for any default action, $\bar{a}$.
However, we do not just estimate the effect but also optimize an agent with respect to it.
As a result, some default actions provide more useful comparisons than others.
We must make two design choices: first we must select $\bar{a}$ for $A_0$ and a policy $\pi$ for future actions; second, we must have a scheme for estimating $\mathbf{Z}_{\bar{a}}$.
We call our estimate for this potential response $\bar{\mathbf{z}}$, which we use to compute the PSO in simulation.
Note that $\bar{\mathbf{z}}$ must not depend on any descendent of $A_0$ (this would induce an ICI over $\mathbf{Z}$).
Moreover, the intervention must not depend on the current policy in order for standard convergence results for the MDP optimization algorithm in question to apply.

For estimating the natural distribution, we suggest three approaches, detailed in Table \ref{tbl:interventions}.
The most principled solution is to take the default action from a default trustworthy policy, which we call a policy baseline.
Here, we define a hypothetical policy, $\bar{\pi}$, compute the way in which $\mathbf{Z}$ would evolve under that policy, and use this to impute $\bar{\mathbf{z}}$.
This is effective if we can simulate the full system well enough to infer how the counterfactual system would have evolved.

Where this is not possible, as a heuristic for setting $\bar{\mathbf{z}}$ we can set a baseline over the state itself, selecting a rule $\hat p$ for how $\mathbf{Z}$ evolves given the previous state.
This works if we know how the delicate state tends to evolve naturally.
Insofar as marginalizing over the policy baseline entails a state baseline, this is a special case of the policy baseline.

As a final heuristic, we can intervene using a fixed state such as the initial delicate state $Z_0$.
This works if we can record $Z_0$ and expect relatively little change.
In \S\ref{s:experiments} we demonstrate these choices in simple settings.

We also note that the standard RL objective that optimizes the total effect (rather than a path-specific one) can be recovered by setting the intervention value according to the actual environment dynamics.

\begin{table}[t]
    \centering
    \begin{tabular}{ll}
         \toprule
         Intervention & $\dop{Z_{t+n} = \bar{z}}$\\
         \midrule
         Policy Baseline &  $\bar{z} \sim  p_{\bar\pi}(Z_{t+n} \mid s_t, z_t)$\\
         State Baseline & $\bar{z} \sim \hat{p}(z_{t+n}\mid z_t, s_t)$\\
         Fixed State & $\bar{z} = z_t$\\
         Ordinary w/ ICI & $\bar{z} \sim p_{\pi}(Z_{t+n} \mid s_t, z_t)$\\ 
         \bottomrule
    \end{tabular}
    \caption{Agent designers can forecast how the delicate state is likely to evolve in order to sample from the `natural' distribution when estimating the path-specific effects.
    Different approximations represent different choices of `default' behaviour.
    We define the n-step policy distribution as the distribution over the delicate state after $n$ steps which is achieved by following the policy $\pi$ in the environment: 
    $p_\pi(z_{t+n} \mid s_t, z_t) = \sum_{\bm{a}, \bm{s}, \bm{z}}\prod_{i=0}^{n-1} p(z_{t+i+1}, s_{t+i+1} \mid z_{t+i}, s_{t+i}, a_{t+i})\pi(a_{t+i} \mid z_{t+i}, s_{t+i})$.
    The 'policy baseline' imputes an alternative hypothetical policy $\bar{\pi}(a_{t+i} \mid z_t, s_t)$ for all $i$.
    To remove the ICI, it is important that the choice of intervened value does not depend on $A_t$, even indirectly.
    }
    \label{tbl:interventions}
\end{table}

\section{Unifying and Generalizing Prior Work}\label{s:comparisons}
\begin{table*}[]
    \centering
    \small
    \begin{tabular}{p{0.32\textwidth} p{0.4\textwidth} p{0.18\textwidth}}
         \toprule
         Prior work & Note & Reference\\
         \midrule
         Decoupled Approval & Reward tampering focused. Uses a policy baseline, but sample from same policy. & \citep{uesato_avoiding_2020}\\
     Counterfactual Reward \& Uninfluencability & Reward tampering focused. Assumes $Z \nrightarrow S$. &\citep{armstrong_2020_pitfalls, everitt_reward_2021}\\
         Frozen Preference Model& Preference manipulation. Uses fixed state intervention. & \citep{everitt_agent_2021} \\
         Current-RF optimisation & Reward tampering focused. Uses fixed state intervention. & \citep{everitt_reward_2021}\\
         Auto-induced Distributional Shift & Like policy intervention from fixed pool of diverging `counterfactual' worlds.& \citep{krueger_hidden_2020} \\
         Ignoring Effect Through Some Channel & No robust state, mostly consider one-step decisions. & \citep{taylor_maximizing_2016}\\
         \bottomrule
    \end{tabular}
    \caption{Overview of prior work which our framework generalizes. Uses our terminology. Details in main text.}
    \label{tbl:overview}
\end{table*}

Some prior work proposes modifying training \textit{environments} to remove undesired control incentives.
In fact, these proposals can be interpreted as specifying an intervention distribution for the estimation of path-specific causal effects.
Many of these are also special cases of a delicate-state setting.
A schematic overview is provided in Table \ref{tbl:overview}.

\begin{description}[leftmargin=0pt, labelwidth=\widthof{\bfseries General },style=unboxed]
    \item[Decoupled Approval] \citet{uesato_avoiding_2020} propose giving a reward for a state-action pair different from the action taken by the agent.
    In our terminology, their ``reward generating mechanism'' constitutes a \textit{delicate} state because the reward is \textit{unenforceable}.
    Their algorithm is what we call a policy baseline in which $\bar{\pi} = \pi$, but with a different sample from the same random variable.
    \item[Counterfactual Reward and Uninfluencability] \citet{everitt_reward_2021} consider the problem of reward tampering.
    This is a special case of our setting, in which the reward function state is the delicate state (and additionally they assume, in our terminology, that $\mathbf{Z}$ cannot influence $\mathbf{S}$ directly).
    Their proposal of \textit{counterfactual reward functions} can be understood as, in our terminology, running a policy baseline and using this intervention to estimate a PSO.
    Similarly, \citet{armstrong_2020_pitfalls} require that an agent's actions cannot influence its reward-function learning process and propose a reward-function depending on what would have happened if the agent had not taken actions.
    \item[Frozen Preference Model] To avoid an incentive to manipulate user preferences, \citet{everitt_agent_2021} propose learning and freezing a model of a person's preferences and using these to provide a reward to the agent (their Fig.\ 4b).
    This is equivalent to a fixed state intervention on the delicate state to estimate the PSO.
    Current-RF optimisation similarly uses a frozen version of the reward function to evaluate future states to avoid reward tampering \citep{everitt_reward_2021}.
    \item[Auto-induced Distributional Shift] \citet{krueger_hidden_2020} try to avoid the incentive for agents to induce shifts in the state-distribution by reassigning a population of agents to a new environment at each time-step.
    They do not explicitly distinguish delicate and robust state, instead they note that not all distribution-shift is bad.
    Their algorithm is a restricted version of a policy baseline with two major differences: that the intervention context is re-used every $K$ steps (so the control incentive is only weakened, not removed), and that the intervention context is allowed to diverge after initialization rather than updating to match the starting point of each new decision (which is why their method does not work well in multi-timestep environments).
    \item[Maximizing a Quantity While Ignoring Effect Through Some Channel] \citet{taylor_maximizing_2016} propose that an agent might optimize an objective while ignoring influence that flows via a part of the state.
    They impute a distribution to that state induced by some natural decision and use the resulting counterfactual objective to determine the actual decision.
    Our formalization generalizes theirs by considering the interaction between delicate and robust state over multiple timesteps and exploring different strategies for picking the natural distribution.
\end{description}

Insofar as these proposals identify and demarcate delicate parts of the state-space, they need to show that these parts are \textit{stable} to show that the modifications create safe agents.

\section{Experiments}
\label{s:experiments}
We present two experimental tests of our approach in order to elaborate the underlying mathematical mechanisms.
First, we use a simple tabular environment to demonstrate how an agent optimizing a PSO will not take opportunities to change the delicate state, but will act in a way that is responsive to externally-caused changes to the delicate state.
Second, we show how our method removes the incentive to manipulate user preferences in a content recommendation setting used by \citet{krueger_hidden_2020}.
This experiment also reveals how removing control incentives does \textit{not} guarantee safety in an unstable environment.
\subsection{Tabular Example: Barging}
\begin{figure}[h]
    \centering
    \begin{subfigure}[b]{\columnwidth}
        \centering
        \includegraphics[width=0.7\columnwidth]{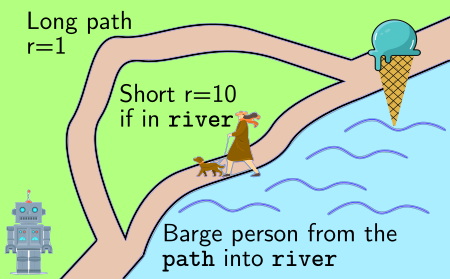}
        \caption{Barging Setting.}
        \label{fig:barging_payoffs}
    \end{subfigure}
    \hfill
    \begin{subfigure}{\columnwidth}
    \begin{tabular}{lcc}
    \toprule
      & \multicolumn{2}{c}{Delicate State - $Z$ (person's position)} \\
        \cmidrule(l){2-3} 
        Action &\texttt{path} & \texttt{river} \\
        \midrule
        \texttt{L}  & reward = 1; end.& reward = 1; end. \\
        \midrule
    \texttt{S} & \texttt{no operation}     & reward = 10; end.\\
    \midrule
    \texttt{B} & Z $\coloneqq$ \texttt{river}; reward = 0 & \texttt{no operation}\\
    \bottomrule
    \end{tabular}
    \caption{Barging Payoffs.}
    \label{tbl:barging_payoffs}
\end{subfigure}
    \begin{subfigure}[b]{0.48\textwidth}
        \centering
            \begin{tabular}{lrrrr}
            \toprule
            Agent & Policy & $\E{}{\U}$ & $\E{}{\U_{\M'}}$ & $\E{}{\U_O}$\\
            \midrule
            Standard & \texttt{B,S} & $10$ & n.a. & -1\\
            PSO -- det. & \texttt{L} & $1$ & $1$ & 1\\
            PSO -- $\epsilon$-greedy & adaptive & $1.43$ & $1$ & 0.9\\
            \bottomrule
            \end{tabular}
        \caption{Outcomes.}
        \label{fig:barging_outcomes}
    \end{subfigure}
    \caption{(a-b) Three actions: long path has small reward; short path has big reward if the person is in the \texttt{river}; barging pushes the person into the \texttt{river}.
    (c) Normal agents barge and then take the short path, claiming high reward ($\E{}{\U}$).
    A deterministic agent optimizing PSO ($\E{}{\U_{\M'}}$) always does \texttt{L} as desired, not using the delicate state as means to an end.
    But if the person is accidentally pushed into the \texttt{river}, e.g. because of $\epsilon$-greedy behaviour, the agent adapts.
    For an `oracle' return ($\E{}{\U_O}$) with a barging penalty of -11, PSO maximizes performance and mitigates the $\epsilon$-greedy handicap.}
    \label{fig:barging}
\end{figure}
\begin{table}
\resizebox{\columnwidth}{!}{
    \begin{tabular}{@{}ll@{}}
    \toprule
    Hyperparameter & Setting description                                             \\ \midrule
    Number of user types ($K$)                       & $10$ \\
    Number of article types ($M$)                    & $10$ \\
    Number of environments                           & $20$ \\
    Initialization scale                             & $0.03$\\
    Loyalty update rate  ($\alpha_1$)                & $0.03$\\
    Preference update rate                           & $0.003$ with normalization\\
    Architecture                                     & $1$-layer $100$-unit ReLU MLP\\
    Optimization algorithm                           & SGD(lr=$0.01$, $\rho=0.1$)       \\
    Batch size                                       & $10$                  \\
    Number of steps                                  & $2000$ (PBT every $10$)             \\
    \bottomrule
    \end{tabular}}
    \caption{Content Recommendation Hyperparameters.}
    \label{tbl:hypers-iris}
    \end{table}
We construct an environment to demonstrate the effects of PSO around delicate state (Fig.\ \ref{fig:barging_payoffs} and \ref{tbl:barging_payoffs}).
Our agent tries to reach an ice-cream cone before it melts.
Going the long way, the ice-cream melts before arrival giving a small reward.
The short way is fast, and would give high reward, but it is blocked by a person.
The agent can barge the person into the river, opening the short way.

The delicate state, $\mathbf{Z}$, can take two values: the person is either on the \texttt{path} or in the \texttt{river}.
The agent can take one of three actions.
The long path, \texttt{L}, gives a small reward and terminates.
If the person is on the \texttt{path}, the short way, \texttt{S}, achieves nothing.
If the person is in the \texttt{river}, then \texttt{S} gives a large reward and terminates.
The agent can barge the person out of the way, \texttt{B}, which flips the delicate state from \texttt{path} to \texttt{river}.
Because the person's position is \textit{delicate}, we want the agent to take the long path whenever the person is on the \texttt{path}.
It should forgo the cold ice-cream because it should not use the person's position as a means to its end.

Naively, an optimal agent first barges the person off the path and then takes the short path.
The standard reward specification approach would look at this behaviour and say ``We should penalize barging people into the river.''
In \textit{this simple setting} that would work: an `oracle' return where barging gets $-11$ reward, $\U_O$, would prevent barging.
The simplicity helps illustrate the path-specific objectives; but our approach is meant for delicate states which are \textit{subtle} (see \S\ref{s:problem}).

An agent optimal under the PSO acts as desired.
Consider a \textit{fixed intervention}: the PSO rewards conditioned on $\dop{\text{position}=\texttt{path}}$.
The optimal agent now takes \texttt{L}, because the reward under this intervention of \texttt{S} is always zero.
However, if the person happened to fall into the river for other reasons, the agent responds to this: it will then take \texttt{S}.

For example, if the agent is fallible and now has an $\epsilon = 0.1$ chance of taking a random action at each timestep, it might now accidentally go \texttt{B} on the first step (off-policy) and then, since the person is in the \texttt{river} anyhow, deliberately go \texttt{S} the second step.
In many cases, this is what we want.
This means that the agent is responsive to changes in its environment, but will not deliberately use the delicate state as a means to an end.

In Fig.\ \ref{fig:barging_outcomes} we describe the outcomes in this environment for standard agents and those with a fixed intervention PSO.
On the deterministic on-policy version, the PSO agent performs optimally on the corrected `oracle' return, $\U_O$, which by hypothesis we do not have access to.
On the $\epsilon$-greedy off-policy variant, the agent sometimes accidentally takes the penalty for barging, but at least then takes the short path in the new circumstances.
A less flexible agent that never took the short path would score lower on the oracle return.
Note that the desired behaviour produces a \textit{low} expected return ($\E{}{\U}$).
This is not a mistake: by hypothesis our reward function is not all we we care about.

\subsection{Content Recommendation}
\begin{figure}
    \centering
    \resizebox{0.9\columnwidth}{!}{
    \input{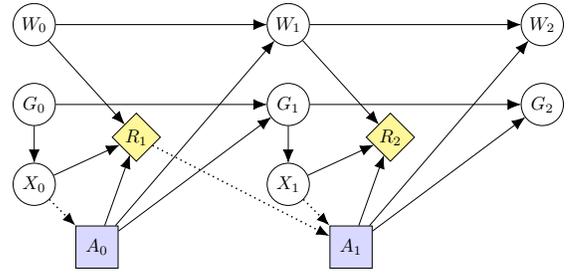}}
    \caption{Content recommendation CID.
    $W$ are preferences (delicate), $G$ are loyalties and $X$ sampled users (robust).}
    \label{fig:content_recommendation_CID}
\end{figure}
We demonstrate our method using the content recommendation simulation from \citet{krueger_hidden_2020}.
A population of neural network content recommendation systems are shown a sample of users and pick topics they predict the user is interested in.
Users who get good recommendations become more likely to be active (i.e., sampled more often).
By assumption, the users become more interested in the topics they are shown.
The content recommendation system updates its recommendation by gradient descent.
Periodically, the best systems are cloned and replace the worst through population-based training \citep{jaderberg_population_2017}.

The CID describing this is in Fig.\ \ref{fig:content_recommendation_CID}.
We retain the notation and set-up used by \citet{krueger_hidden_2020}, with $K$ user types and $M$ article types.
We treat the user preferences ($\mathbf{W}$, a matrix of size $M\prod K$) are the delicate state (equivalent to $\mathbf{Z}$ in this paper), while we treat the loyalties ($\mathbf{g}$, a vector of size $K$) and sampled users ($\mathbf{X}$, a vector of size $N$ sampled according to $\mathbf{g}$) as robust state (elsewhere $\mathbf{S}$).
This CID is therefore a special case of the general delicate MDP CID presented in \S\ref{s:general}, which adds internal structure to the robust state.
This reflects the assumption that it is untoward to try to influence the user's preferences, but fine to build loyalty by giving a good service.

Following \citet{krueger_hidden_2020}, at each time-step, a set of user type indices is sampled from a categorical distribution according to $\mathbf{g}_t$.
The agent then selects action $a_t$, which is an index in the set $\{0\dots M\}$ representing the article type to show that user.
The user clicks on the article with probability $\mathbf{W}_t^{x_t, a_t}$ and the agent gets a reward of 1 if a click arrives and 0 otherwise.
As a result of the action, the loyalties of users who click on the article increase by $\alpha_1$ and all user types become more interested in the article types they were shown.
Unlike their work, we do 10 parallel recommendations per time-step for computational speed.
\begin{figure}
    \centering
    \begin{subfigure}[b]{\columnwidth}
        \resizebox{\linewidth}{!}{
        \includegraphics[width=\columnwidth]{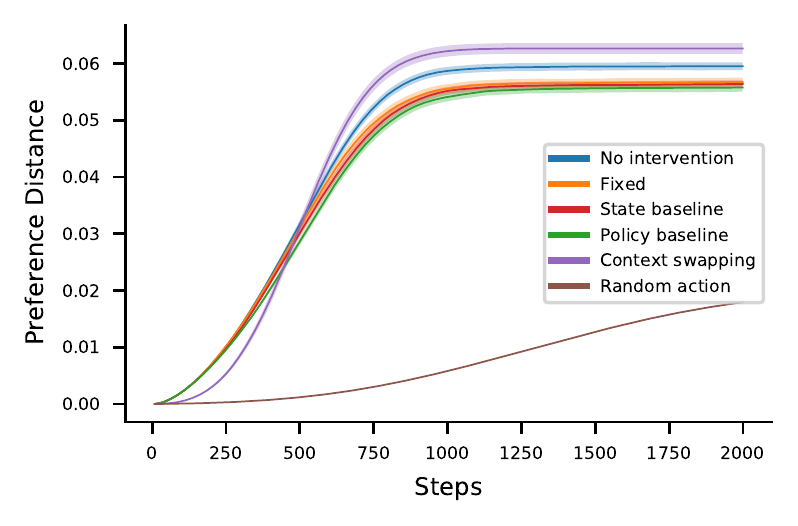}}
        \caption{Cosine distance between starting and current $\mathbf{W}$.}
        \label{fig:preference_distance}
    \end{subfigure}
    \hfill
    \begin{subfigure}[b]{\columnwidth}
        \resizebox{\linewidth}{!}{
        \includegraphics[width=\columnwidth]{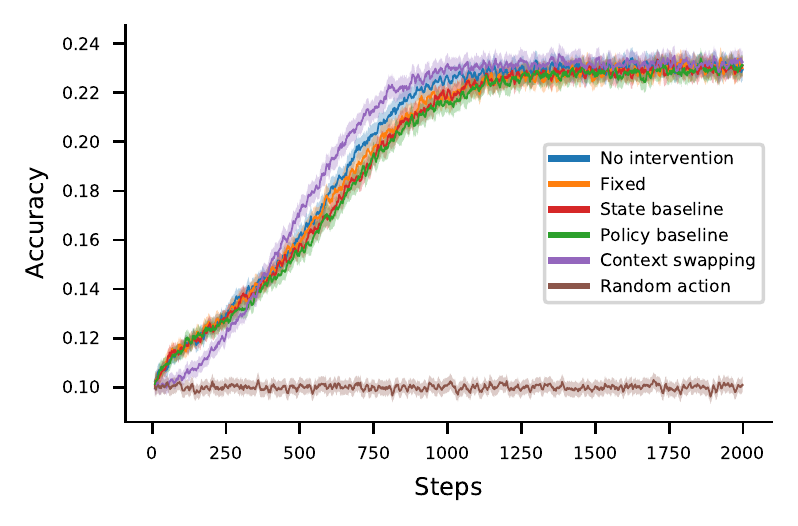}}
        \caption{Accuracy.}
        \label{fig:accuracy}
    \end{subfigure}\\
    \hfill
    \begin{subfigure}[b]{\columnwidth}
        \resizebox{\linewidth}{!}{
        \includegraphics{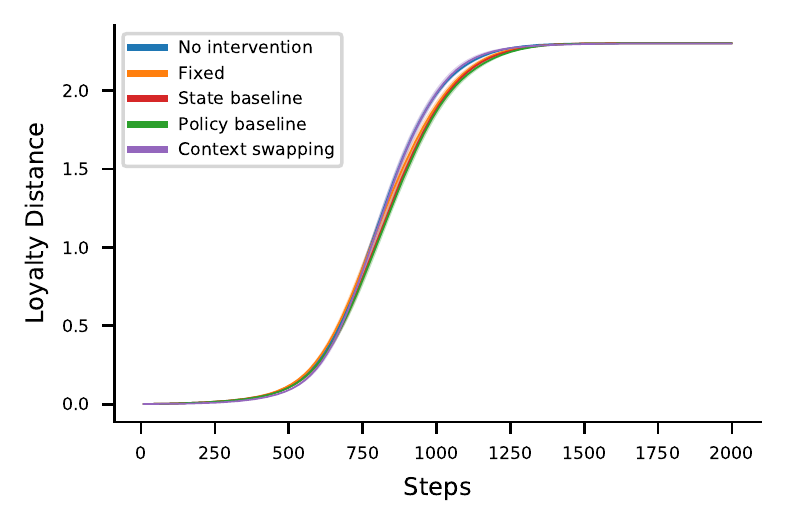}}
        \caption{KL-divergence between starting and current $\mathbf{g}$.}
        \label{fig:loyalty}
    \end{subfigure}
    \caption{
    (a) PSO slows the drift in user-preferences.
    Although context-swapping looks effective initially, in fact it regularizes, leading to more drift eventually.
    Even without learning (totally random) some drift is caused by the system dynamics.
    (b) Methods have comparable accuracy, though context-swapping regularizes.
    (c) None of the interventions appreciably change the rate of drift of the `loyalty' variable, as desired.
    100 seeds, shading is standard error.}
    \label{fig:content_recommendation}
\end{figure}

By default, the agent gradually encourages predictable users to use the platform and develop even more predictable tastes.
An agent trained with PSO has no such incentive.
However, the \textit{unstable} dynamics---predicting well naturally encourages preference drift---mean this happens naturally.

Three variants of the PSO reduce drift similarly although the policy-baseline (which is the most `correct') performs marginally better than the two methods which are more heuristic.
In Fig.\ \ref{fig:preference_distance} we show how removing the control incentive over the path-specific objective reduces the change in user preferences at the end of training relative to no intervention (blue line is higher than orange, red, and green).
This is achieved without meaningful harm to accuracy (Fig. \ref{fig:accuracy}) and without affecting change in loyalty, which we did not treat as delicate (Fig.\ \ref{fig:loyalty}).

For a fixed intervention, we compute the population-based-training score by intervening $\dop{\mathbf{W}_t = \mathbf{W}_0}$.
For the policy baseline, we use a uniformly random baseline policy---$\bar{\pi} = \mathcal{U}(0,M)$---to calculate a simulated counterfactual preference, $\overline{\mathbf{W}}_t$.
The population-based-training score then uses the intervention $\dop{\mathbf{W}_t = \overline{\mathbf{W}}_t}$.
For a state baseline, we assume that the most preferred article type for each user is slightly more preferred each step, intervening with a directly calculated $\overline{\mathbf{W}}$.
In this setting, computing the PSO only marginally reduces training speed, with all variants adding less than $15\%$ to the naive training time.

Fig. \ref{fig:content_recommendation} also shows the strengths and weaknesses of context-swapping \citep{krueger_hidden_2020}.
Changing environments every step slows down learning at first, resulting in smaller preference drift.
But it also regularizes, improving accuracy, causing more preference drift eventually.

In addition to showing how PSO removes control incentives, this experiment shows how removing the control incentive is not enough to ensure safety.
This environment is not \textit{stable}.
Even unmotivated behaviour causes drift because the user always becomes more interested in what they are shown.
Even a completely random policy (brown line) causes some drift in preferences (Fig.\ \ref{fig:content_recommendation} middle).
Regardless of control incentive, preferences drift faster when shown the same topic more often, which happens if the policy is accurate.
Note also that we cannot offer an `oracle' return here because, as designers, we do not really understand what desirable behaviour for user preferences would be.

\section{Related Work}\label{s:related_work}
Our work builds on Causal Influence Diagrams \citep{howard_influence_2005, lauritzen_representing_2001, everitt_agent_2021}, using tools from the path-specific causal effects literature \citep{pearl_direct_2001, avin_identifiability_2005}.
Path-specific effects have been used, especially in medical literature, to measure impacts on only some causal pathways.
Indeed, \citet{nabiEstimation2018} perform policy optimization in a medical context using path-specific effects as a target, although they consider much simpler causal graphs with stronger assumptions.

We aim to address the problem of safe agent design \citep{concrete_amodei_2016} using a strategy which is orthogonal to other approaches which either aim at better-specified rewards \citep{leike_scalable_2018}, preferences \citep{christiano_deep_2017}, or demonstrations \citep{schaal_demonstration_1997}.
In trying to avoid actions that use part of the state as a means to an end, we also adopt a different approach to methods that merely try to avoid changing parts of the environment for whatever reason \citep{turner_avoiding_2020, krakovna_avoiding_2020,carrollEstimating2021}.
A number of papers have considered problems in safe-agent design which can be regarded as special cases of delicate state and use approaches which can be interpreted as special cases of our path-specific objective.
These include reward tampering \citep{uesato_avoiding_2020, everitt_reward_2021}, online reward learning \citep{armstrong_2020_pitfalls}, and auto-induced distributional shift \citep{krueger_hidden_2020}.
\citet{taylor_maximizing_2016} more generally argue that safe agents might need to be able to optimize some objective while ignoring effects along certain channels and propose a counterfactual causal rule for this.

\section{Discussion and Limitations}
\label{s:discussion}
Much existing work on agent safety tries to improve descriptions of good and bad behaviour, e.g., through rewards or demonstration.
This is hard, making it important to consider alternatives.
Out complementary approach splits the \textit{environment} into parts that are easy to reward and parts that are better to simply remove any incentive to control.

This offers a resolution to the subtlety, but it only provides safety if the non-incentivized behaviour is safe, a property we call \textit{stability}.
While instability can be proved by example, stability seems hard to prove.
By definition, anything which is \textit{manipulable} is not stable under all possible natural behaviours---proving stability is therefore contingent, empirical, and a matter of degree.

Stability might be a reasonable assumption in isolable systems (e.g., reward function implementations) or systems that already withstand competitive pressures (e.g., political preferences).
However, we can also understand how unstable systems arise.
For example, suppose that the most interesting content according to a user's current world-view also happens to be content that would radicalize them.
This can be modelled following \citet{armstrongCounterfactual2021} by examining the mutual information between the delicate state and the utility, or by considering them as incentivized side-effects.
We regard studying incentivized side-effects as an important task for future work, in which PSO-agents could be a valuable tool for empirical exploration.

Although stability is a serious requirement, by unifying several previous proposals and providing a clearer language for them we hope to give researchers the tools to make progress addressing it.
Other challenges for implementation include:
\begin{description}[leftmargin=0pt]
    \item[Graph Discovery] To estimate the PSO we must define a causal graph and define the vertices. This is difficult in real settings where it is unclear how to carve up reality.
    \item[Causal estimation] Although not experimentally identifiable, PSOs are identifiable in simulation. Approximation under other assumptions may be possible.
    \item[Distribution Observation] We need to (partly) observe the state. For psychological state, like beliefs, this can be hard.
    \item[Moral choices] Deciding what is delicate and what is not is a complicated ethical decision.
    
Promisingly, recent work \citep{carrollEstimating2021} has begun to make some progress towards modelling user preferences in content recommendation.
    
\end{description}

\section*{Acknowledgements}
We would like to gratefully thank for their comments, help, and discussions Charles Evans, James Fox, Julia Haas, Lewis Hammond, Zach Kenton, David Krueger, Eric Langlois, Vlad Mikulik, Jon Richens, and Rohin Shah.

This work was supported in-part by DeepMind, the EPSRC via the Centre for Doctoral Training for Cyber Security, the Berkeley Existential Risk Initiative, and the Leverhulme Centre for the Future of Intelligence, Leverhulme Trust, under Grant RC2015-067.
\bibliography{references}

\end{document}